\documentclass[a4paper]{styles/svproc}

\usepackage{url}

\usepackage{newtxtext,newtxmath} 
\usepackage[T1]{fontenc}
\usepackage[utf8]{inputenc}

\usepackage{amsmath,amssymb}
\usepackage{graphicx}
\usepackage{booktabs}
\usepackage{multirow}
\usepackage{microtype}
\usepackage[table]{xcolor}
\usepackage{algorithm}
\usepackage{algpseudocode}
\usepackage{hyperref}
\usepackage{cleveref}
\usepackage[inline]{enumitem}
\usepackage{cleveref}
\usepackage{cite}
\usepackage{url}
\usepackage{subcaption}
\usepackage[font=small]{caption}

\newcommand{\Path}{\zeta}
\newcommand{\Paths}{\Omega_H}
\newcommand{\Belief}{b}
\newcommand{\Risk}{R}
\newcommand{\InfoS}{I_{\mathrm{S}}}
\newcommand{\InfoB}{B_{\alpha}}
\newcommand{\EntropyS}{H_{\mathrm{S}}}
\newcommand{\EntropyB}{H_{\mathrm{B}}}

\newcommand{\E}{\mathbb{E}}
\newcommand{\argmax}{\operatorname*{arg\,max}}

\begin{document}
\mainmatter              
\title{When Information is Worth the Risk:\\ Behavioral Valuation for Hazardous Robotic Exploration}
\titlerunning{When Information is Worth the Risk}
\author{
Alkesh K. Srivastava\inst{1}\and
Aamodh Suresh\inst{2,3} \and
Carlos Nieto-Granda\inst{2} \and
Philip Dames\inst{1}
}
\authorrunning{A.K. Srivastava et al.}
\institute{
Department of Mechanical Engineering, Temple University, Philadelphia, PA 19122, USA\thanks{
A. Srivastava and P. Dames were supported by NSF grant~CNS-2143312. Part of this work was conducted during A. Srivastava's internship at the U.S. Army Research Laboratory.}\
\email{{alkesh,pdames}@temple.edu}
\and
U.S. DEVCOM Army Research Laboratory, Adelphi, MD 20783, USA\\
\email{carlos.p.nieto2.civ@army.mil}
\and
Lobster Robotics, The Hague, South Holland, Netherlands\thanks{A. Suresh is currently with Lobster Robotics. This work was conducted while he was with the U.S. Army Research Laboratory.}\\
\email{aamodh@gmail.com}
}

\maketitle   

\begin{abstract}
Hazardous robotic exploration requires robots to map spatial risks, such as unsafe terrain, radiation, fire, mines, or structural damage, while operating where collecting information can itself cause failure. A highly informative path may expose the robot to hazards, terminate execution, and prevent future observations. Hazardous exploration therefore requires deciding not only where uncertainty is largest, but when reducing it is worth the risk. This paper introduces a valuation-layer view of this problem. We keep the belief update, sensor model, physical risk model, and finite-horizon informative planner fixed, and change only the scalar objective used to rank feasible paths. Within this framework, we introduce a risk-augmented Behavioral Information objective based on Prelec probability weighting, yielding an interpretable family of conservative-to-aggressive information-risk valuations. Theoretically, we show that valuation parameters create switching boundaries between high-information/high-risk and lower-information/lower-risk paths, and induce a transformed Pareto-frontier structure over feasible exploration policies. Large-scale failure-truncated grid-world experiments show that valuation alone reshapes the information-risk frontier. Shannon information planning remains a strong raw-information baseline, while risk-aware objectives can reduce hazard exposure and robot losses by avoiding failures that truncate future sensing. Risk-augmented Behavioral valuation is Pareto-competitive with standard risk-aware baselines and provides interpretable conservative and intermediate regimes. These results support a framework in which robots reason not only about how much uncertainty an action reduces, but whether that reduction is worth the risk required to obtain it.
\end{abstract}


    
\section{Introduction}
\label{sec:introduction}

Robots are often deployed for the ``three Ds'' of robotics: tasks that are dull, dirty, or dangerous. This paper focuses on the dangerous case, where human access is unsafe and the hazards are initially unknown. Examples include earthquakes and building collapses, mines and tunnels, fires or chemical leaks, damaged nuclear infrastructure, planetary surfaces, and other high-risk environments. In these settings, exploration is not only about building a map of the world; it is about building a map of risk. A robot may need to identify hazards while simultaneously avoiding the regions it is trying to understand. This creates a fundamental tension: the most informative path may also be the path most likely to destroy the robot. Entering a poorly understood region can reduce uncertainty, but it can also terminate the mission, prevent future observations, and eliminate the robot as a sensing resource. Hazardous exploration therefore raises a question that is largely absent from benign information gathering: \emph{when is information worth the risk of acquiring it?}

Information-theoretic planning provides a principled way to select actions that reduce uncertainty. A robot can choose paths that maximize expected entropy reduction or mutual information under a probabilistic sensor model. In safe environments, this abstraction is often appropriate because the value of an observation is determined primarily by how much it improves the belief. In standard informative planning, acquiring data already incurs costs such as motion, time, or energy. In hazardous environments, however, the acquisition cost can be catastrophic and survival-coupled. Two paths with similar expected information gain can have very different mission value if one is likely to cause failure, truncate future sensing, or expose the robot to severe hazards. 

Risk-aware planning commonly addresses this issue by adding constraints or penalties to an information-seeking objective. Such approaches are important, but they also reveal a deeper conceptual distinction. Hazardous exploration is not only a problem of choosing feasible paths; it is also a problem of valuing those paths under risk. The Bayesian filter determines what the robot believes, the planner determines what actions are feasible, and the valuation functional determines when a feasible information-gathering action is worth taking. We isolate this valuation layer by keeping the belief update and planner fixed while changing only how candidate paths are scored under risk. This lets us ask how information-risk valuation alone reshapes exploration behavior.

\begin{figure}[t]
    \centering
    \includegraphics[width=\linewidth]{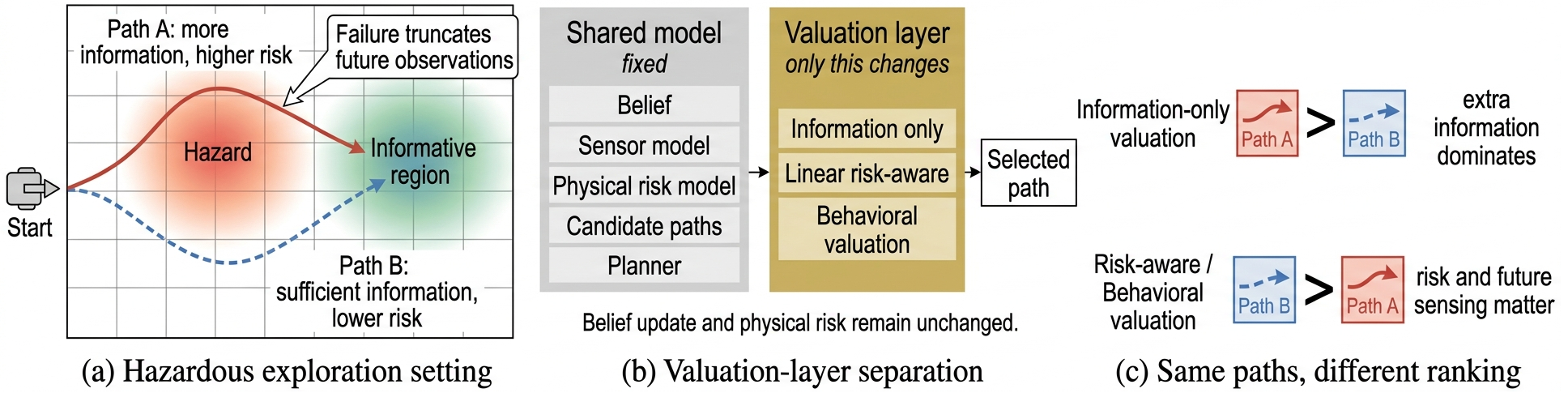}
   \caption{Hazardous exploration as a valuation problem. (a) Candidate paths trade off expected information and physical risk under failure-truncated execution. (b) The belief update, sensor model, risk model, candidate paths, and planner are fixed; only the valuation layer changes. (c)~Different valuation objectives can rank the same candidate paths differently.}
    \label{fig:concept}
\end{figure}

\Cref{fig:concept} summarizes the valuation-layer view: the belief update, sensor model, physical risk model, candidate paths, and planner remain fixed, while only the objective used to rank information-risk tradeoffs changes. We instantiate this layer with a risk-augmented Behavioral Information objective based on Prelec probability weighting, applied only at the objective level so Bayesian beliefs and physical risk probabilities remain unchanged. {Here, Behavioral Information Gain is used as a planning valuation functional rather than as a replacement for Shannon mutual information.} This yields an interpretable nonlinear information-risk tradeoff; we derive path-switching boundaries and a transformed Pareto interpretation, then show that valuation alone reshapes exploration behavior in large-scale failure-truncated experiments.

The contributions of this paper are fourfold. First, we formulate hazardous robotic exploration as a valuation problem layered on top of standard Bayesian belief updates and finite-horizon planning. Second, we introduce a risk-augmented Behavioral Information objective based on Prelec probability weighting, providing an interpretable nonlinear family of information-risk objectives. Third, we derive path-switching conditions and interpret the objective as a scalarization of a transformed information-risk tradeoff set. Fourth, we evaluate the framework in large-scale failure-truncated experiments with standard risk-aware baselines, showing that valuation alone exposes a nontrivial information-risk frontier.

\section{Related Work}
\label{sec:related_work}

This work lies at the intersection of informative path planning, failure-aware exploration, risk-aware planning, behavioral probability weighting, and multiobjective decision making. We review these areas to position the main distinction of this paper.

\paragraph{Informative Path Planning and Active Sensing:}
Informative path planning and active sensing select robot actions to reduce uncertainty about an unknown environment, often using entropy reduction or mutual information. Foundational work established information gain as a planning objective for mobile robots~\cite{bourgault2002information}, followed by information-theoretic control and distributed sensing for robotic sensor networks~\cite{julian2012distributed,charrow2014approximate}, mutual-information policies for active sensing~\cite{macdonald2019active}, and cooperative search in risky or hazardous environments~\cite{flint2003cooperative,yang2007multi,tomo,schwager2017multi}. These works provide principled methods for deciding where observations are valuable. However, standard informative planning often treats path cost through deterministic budgets, expected costs, or task-specific risk models. In hazardous exploration, information acquisition may be probabilistic and catastrophic: the robot may fail before collecting future observations. Our work keeps the informative planning substrate fixed and studies how candidate paths should be valued when expected information gain and physical failure risk are coupled.

\paragraph{Failure-Truncated and Path-Based Sensing:}
A closely related line of work studies information gathering when sensing is failure-truncated. Otte and Sofge introduced \emph{path-based sensors}, where a sensor reports whether an event occurred somewhere along a path without localizing where it occurred~\cite{Otte.Sofge.TASE21}. Subsequent work extended this idea to distributed multi-robot search in communication-denied hazardous environments and broader Bayesian-network formulations~\cite{srivastava2022distributed,srivastava2026bayesian}. Unlike path-based sensing, our model uses cell-level binary hazard observations along executed paths. However, both settings share the same coupling: failure truncates future sensing. We study the complementary valuation question of how to rank candidate paths when expected information and failure risk compete under a fixed sensing and planning model.

\paragraph{Risk-Aware and Survivability-Constrained Planning:} Risk-aware planning addresses hazards by constraining or penalizing risky actions. Survivability-aware routing problems, such as the Team Surviving Orienteers problem, maximize coverage or reward subject to survival constraints~\cite{jorgensen2018team}, and related work studies survivability and diversity in expendable robot teams~\cite{lyu2016k}. Recent risk-aware informative path planning also incorporates explicit risk measures, including CVaR-based viewpoint selection for 3D surface inspection~\cite{xu2025risk}. These methods explicitly regulate risk, but primarily through constraints, penalties, or risk measures. In contrast, our framework studies risk valuation as an objective-layer mechanism and compares directly against risk-penalized and chance-constrained Shannon baselines.
\paragraph{Behavioral Probability Weighting and Behavioral Entropy:}
Prelec probability weighting and related models from cumulative prospect theory represent nonlinear probability valuation~\cite{prelec1998probability}. In robotics, such weighting has been used for behavioral path planning under uncertain risks and rewards~\cite{suresh2023purr}. Behavioral Entropy extends this idea to exploration by replacing Shannon entropy with a Prelec-weighted uncertainty functional~\cite{suresh2024robotic}, with follow-on work in adaptive source seeking and behavior-adaptive hazard localization~\cite{ghimire2025beasst,srivastava2025behaviorally}. These works motivate the use of probability-weighted uncertainty, but they do not isolate the explicit single-agent tradeoff between Behavioral Information Gain and physical failure risk under a fixed planner, Bayesian update, and physical risk model. Our contribution is a risk-augmented Behavioral Information objective that applies probability weighting only at the valuation layer, leaving the belief update, sensor model, and physical risk probabilities unchanged.

\paragraph{Multiobjective Planning and the Valuation-Layer Gap:}

The information-risk tradeoff can also be viewed through the lens of multiobjective planning. Pareto Monte Carlo Tree Search addresses multiobjective informative planning by searching over competing objectives such as exploration and exploitation~\cite{Chen2019ParetoMC}. Recent work on multiobjective task assignment studies how to compute statistically distinct plans that represent different tradeoffs under uncertainty~\cite{wilde2024statistically}. Our setting similarly involves competing objectives, but the focus is different. Rather than developing a new multiobjective search algorithm or attempting to enumerate a Pareto set, we study how a scalar valuation functional transforms the risk-information geometry of a fixed feasible path set. This gives the paper its central positioning---\emph{hazardous exploration is not only a problem of estimating hazards or planning feasible paths, but of valuing information when acquiring it can cause failure.}

\section{Hazardous Exploration as a Valuation Problem}
\label{sec:formulation}

We consider a robot operating in an unknown hazardous environment. The robot's mission is to infer the spatial distribution of static hazards while selecting paths whose execution may expose it to failure. Each cell \(c\in\mathcal{C}\) has a binary latent hazard state \(Z_c\in\{0,1\}\), where \(Z_c=1\) indicates the presence of a hazard, and an associated lethality \(\ell(c)\in[0,1]\), which determines the probability of failure if the robot encounters that hazard. Thus, a candidate path has both an information value, through the observations it may provide, and a physical risk, through the hazards it may traverse. The central question is how such paths should be valued when information gain and failure risk are~coupled.

Let \(x_t\in\mathcal{X}\) denote the robot state at time \(t\), and let \(M=\{Z_c\}_{c\in\mathcal{C}}\) denote the latent hazard map. The robot maintains the Bayesian marginal belief as \(\Belief_t(c)=\Pr(Z_c=1\mid h_t),\) conditioned on the action-observation history \(h_t\). This belief represents the robot's current estimate of where hazards are located. At each decision epoch, the robot evaluates feasible finite-horizon candidate paths. For horizon \(H\), define
\begin{equation}
    \Paths(x_t)=\{(x_t,x_{t+1},\ldots,x_{t+H}) : x_{k+1}\in f(x_k,u_k),\ u_k\in\mathcal{U}\},
\end{equation}
where \(f\) encodes the robot dynamics and grid constraints, and \(\mathcal{U}\) is the action set. A candidate path is denoted by \(\Path=(x_t,x_{t+1},\ldots,x_{t+H})\in\Paths(x_t)\). All compared methods use this same feasible path set; they differ only in how paths are valued.

Executing or evaluating a path $\Path$ 
generates cell-level observations along the cells visited by the path. Let \(\mathcal{V}(\Path)\subseteq\mathcal{C}\) denote the unique cells visited by \(\Path\). For each observed cell \(c\in\mathcal{V}(\Path)\), the robot receives a binary observation \(Y_c\in\{0,1\}\) of the latent hazard state \(Z_c\). The observation model is a noisy binary channel with true-positive rate \(p_{\mathrm{tp}}\) and false-positive rate \(p_{\mathrm{fp}}\). We write \(Y_{\Path}=\{Y_c:c\in\mathcal{V}(\Path)\}\) for the collection of path observations. After observing \(y_{\Path}\), the robot updates its hazard belief using the standard Bayesian update
\begin{equation}
    \label{eq:belief_update}
    \Belief_{t+1}=\tau(\Belief_t,\Path,Y_{\Path}).
    \end{equation}
Here, \(\tau\) is the Bayesian belief-update operator induced by the observation likelihood and the prior belief. 

The information value of a path is measured by expected uncertainty reduction in the hazard belief. For a binary cell belief \(q\in[0,1]\), the Shannon entropy is
\begin{equation}
    H_{\mathrm{S}}(q)=-q\log q-(1-q)\log(1-q),
    \label{eq:binary_entropy}
\end{equation}
with \(0\log0=0\). The total hazard-belief entropy is approximated by the sum of cell entropies,
\begin{equation}
    \EntropyS(\Belief_t)=\sum_{c\in\mathcal{C}}H_{\mathrm{S}}(\Belief_t(c)).
    \label{eq:total_hazard_entropy}
\end{equation}

The Shannon information gain of a path is the expected reduction in this entropy after incorporating the path observations:
\begin{equation}
    \InfoS(\Path;\Belief_t)=\EntropyS(\Belief_t)-\E_{Y_{\Path}\sim p(\cdot\mid \Belief_t,\Path)}\left[\EntropyS(\tau(\Belief_t,\Path,Y_{\Path}))\right].
    \label{eq:shannon_ig}
\end{equation}

Each candidate path also has physical risk. The additive expected exposure is
\begin{equation}
    R_{\mathrm{exp}}(\Path;\Belief_t)=\sum_{c\in\mathcal{V}(\Path)}\Belief_t(c)\ell(c),
    \label{eq:exposure_risk}
\end{equation}
and the corresponding path failure probability is
\begin{equation}
    R_{\mathrm{fail}}(\Path;\Belief_t)=1-\prod_{c\in\mathcal{V}(\Path)}\left(1-\Belief_t(c)\ell(c)\right).
    \label{eq:failure_risk}
\end{equation}
The experiments use \(R_{\mathrm{fail}}\) over unique visited cells as the hazard-risk term. Unless otherwise stated, \(\Risk(\Path;\Belief_t)\) denotes this expected hazard failure probability.

The robot selects a path by maximizing a scalar valuation objective,
\begin{equation}
    \Path^*\in\argmax_{\Path\in\Paths(x_t)}J(\Path;\Belief_t).
    \label{eq:generic_planning_problem}
\end{equation}
The central question of this paper is how the choice of \(J\) changes exploration behavior when the feasible paths, Bayesian update, sensor model, and physical risk model are held fixed. In this way, hazardous exploration is studied as a valuation problem layered on top of ordinary Bayesian inference and finite-horizon planning.

\section{Behavioral Risk Valuation}
\label{sec:behavioral}

This section defines the valuation family used to rank candidate paths under risk. The key modeling choice is that probability weighting is applied only at the objective layer. It changes how information and risk are valued/perceived before path selection.

\subsection{Prelec Probability Weighting}
\label{subsec:4.1}
For \(p\in[0,1]\), \(\alpha>0\), and \(\beta>0\), the Prelec weighting function~\cite{prelec1998probability} is
\begin{equation*}
    w_{\alpha,\beta}(p)=\exp\left[-\beta(-\log p)^{\alpha}\right],
    \label{eq:prelec}
\end{equation*}
with endpoint conventions \(w_{\alpha,\beta}(0)=0\) and \(w_{\alpha,\beta}(1)=1\). The parameter \(\alpha\) controls the curvature of the weighting function, while \(\beta\) controls its nontrivial fixed point; prior work has used this fixed-point structure to choose \(\beta\) so that Behavioral Entropy satisfies admissibility properties as a generalized entropy~\cite{suresh2024robotic}. In this paper, we use Prelec weighting as a path-valuation mechanism rather than as a new probabilistic model. We fix \(\beta=1\) in the experiments and write \(w_{\alpha}(p)\equiv w_{\alpha,1}(p)\), so that \(\alpha=1\) recovers the identity weighting \(w_{\alpha}(p)=p\). {We fix $\beta=1$ to preserve the identity reference $w_{1,1}(p)=p$ and isolate the effects of curvature through $\alpha$ and explicit risk sensitivity through $\eta$; varying $\beta$ would enlarge the valuation family by shifting the weighting function and is left for future study}. As shown in~\Cref{fig:prelec_weighting}, values \(\alpha<1\) overweight small probabilities and produce more conservative valuation of rare hazards, while values \(\alpha>1\) underweight small probabilities and can produce more aggressive valuation.

\begin{figure}[t]
    \centering
    \includegraphics[width=0.5\linewidth]{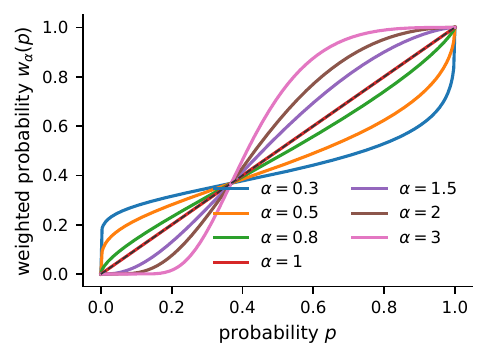}
    \caption{Prelec weighting of scalar path risk for different \(\alpha\) values with \(\beta=1\). The physical risk probability is unchanged; \(\alpha\) changes only the objective-layer valuation of that risk.}
    \label{fig:prelec_weighting}
\end{figure}

\subsection{Behavioral Entropy and Behavioral Information Gain}

Because the hazard belief is represented by cell-wise Bernoulli marginals, we define Behavioral Entropy at the cell level. For a cell with hazard probability \(q\in[0,1]\), the Prelec-weighted binary probabilities are
\begin{equation}
    \tilde{q}_{\alpha}
    =
    \frac{w_{\alpha}(q)}
    {w_{\alpha}(q)+w_{\alpha}(1-q)},
    \qquad
    1-\tilde{q}_{\alpha}
    =
    \frac{w_{\alpha}(1-q)}
    {w_{\alpha}(q)+w_{\alpha}(1-q)} .
    \label{eq:binary_weighted_probs}
\end{equation}
The binary Behavioral Entropy is\(H_{\mathrm{B},\alpha}(q)=-\tilde{q}_{\alpha}\log \tilde{q}_{\alpha}-(1-\tilde{q}_{\alpha})\log(1-\tilde{q}_{\alpha})\), with the convention \(0\log0=0\). The total Behavioral hazard-belief entropy is approximated by the sum of cell entropies, $\EntropyB(\Belief_t) = \sum_{c\in\mathcal{C}} H_{\mathrm{B},\alpha}(\Belief_t(c)).$

The Behavioral Information Gain of a path is defined as the expected reduction in this Behavioral hazard-belief entropy after incorporating the path observations:
\begin{equation}
    \InfoB(\Path;\Belief_t)
    =
    \EntropyB(\Belief_t)
    -
    \E_{Y_{\Path}\sim p(\cdot\mid \Belief_t,\Path)}
    \left[
        \EntropyB\!\left(\tau(\Belief_t,\Path,Y_{\Path})\right)
    \right].
    \label{eq:behavioral_ig}
\end{equation}
The posterior \(\tau(\Belief_t,\Path,Y_{\Path})\) and the expectation over \(Y_{\Path}\) are ordinary Bayesian quantities. Only the entropy functional used to evaluate prior and posterior beliefs is changed; the belief update itself remains unchanged.

\begin{remark}
We use the term Behavioral Information Gain rather than Behavioral Mutual Information. Since probability weighting changes the entropy functional, \eqref{eq:behavioral_ig} should be interpreted as a valuation functional for planning rather than as a replacement theorem for Shannon mutual information.
\end{remark}

\subsection{Risk-Augmented Behavioral Objective}

The central valuation family is
\begin{equation}
    J_{\alpha,\eta}(\Path;\Belief_t)=\InfoB(\Path;\Belief_t)-\eta\,w_{\alpha}\!\left(\Risk(\Path;\Belief_t)\right).
    \label{eq:risk_aug_behavioral}
\end{equation}
Here, \(\Risk(\Path;\Belief_t)\) is the ordinary expected hazard failure probability defined in Sec.~\ref{sec:formulation}. The parameter \(\alpha\) controls nonlinear risk valuation, while \(\eta\ge0\) controls the explicit tradeoff between Behavioral Information Gain and perceived path risk.

This separation is central to the paper. The robot does not update its hazard belief using \(w_{\alpha}\), and sampled failure events are generated from the physical hazard and malfunction models. Prelec weighting affects only the ranking of candidate paths before execution.

The objective includes three useful special cases:
\begin{itemize}
    \item If \(\eta=0\), the objective reduces to pure Behavioral Information Gain.
    \item If \(\alpha=1\), then \(w_{\alpha}(p)=p\), so the risk term becomes a linear penalty on physical path risk.
    \item If \(\eta=0\) and \(\alpha=1\), the normalized Behavioral Information Gain reduces to Shannon information gain in the binary-cell setting used by the planner.
\end{itemize}

\subsection{Baselines}
\label{sec:behavioral:baselines}
We compare \eqref{eq:risk_aug_behavioral} to three standard baselines:
\begin{align}
    J_{\mathrm{S}}(\Path;\Belief_t)&=\InfoS(\Path;\Belief_t), \label{eq:shannon_baseline}\\
    J_{\mathrm{S},\eta}(\Path;\Belief_t)&=\InfoS(\Path;\Belief_t)-\eta\Risk(\Path;\Belief_t), \label{eq:risk_penalized_shannon}\\
    \Path^*&\in\argmax_{\Path\in\Paths(x_t)}\InfoS(\Path;\Belief_t)\quad\text{s.t.}\quad \Risk(\Path;\Belief_t)\le\delta. \label{eq:chance_shannon}
\end{align}
Eq.~\eqref{eq:risk_penalized_shannon} is the most direct risk-aware baseline. When \(\alpha=1\), the risk component of \cref{eq:risk_aug_behavioral} becomes a standard linear penalty on risk. Comparisons with risk-penalized Shannon therefore test what nonlinear Prelec risk valuation adds beyond ordinary linear risk penalization.

\section{Theoretical Analysis}
\label{sec:theory}

This section analyzes the role of valuation in hazardous exploration. The Bayesian update, sensor model, physical risk model, feasible path set, and finite-horizon search procedure are fixed; only the scalar objective \(J_{\alpha,\eta}\) in~\cref{eq:risk_aug_behavioral} changes. We show how this valuation layer induces path-switching boundaries and a transformed information-risk frontier.

\subsection{Valuation-Induced Path Switching}

Consider two feasible paths \(\Path_1,\Path_2\in\Paths(x_t)\). Let
\begin{align}
    \Delta B_{\alpha}&=\InfoB(\Path_1;\Belief_t)-\InfoB(\Path_2;\Belief_t), \label{eq:delta_b}\\
    \Delta W_{\alpha}&=w_{\alpha}\!\left(\Risk(\Path_1;\Belief_t)\right)-w_{\alpha}\!\left(\Risk(\Path_2;\Belief_t)\right). \label{eq:delta_w}
\end{align}
Here, \(\Delta B_{\alpha}\) is the Behavioral Information advantage of \(\Path_1\) over \(\Path_2\), and \(\Delta W_{\alpha}\) is its perceived-risk disadvantage. When both are positive, \(\Path_1\) is more informative but also~riskier.

\begin{proposition}[Risk-information switching boundary]
\label{prop:switching}
Suppose \(\Delta B_{\alpha}>0\) and \(\Delta W_{\alpha}>0\). Then \(\Path_1\) and \(\Path_2\) are indifferent under \(J_{\alpha,\eta}\) at
\begin{equation}
    \eta^*(\alpha)=\frac{\Delta B_{\alpha}}{\Delta W_{\alpha}}.
    \label{eq:eta_star}
\end{equation}
For \(\eta<\eta^*(\alpha)\), the higher-information path \(\Path_1\) is preferred. For \(\eta>\eta^*(\alpha)\), the lower-risk path \(\Path_2\) is preferred.
\end{proposition}

\begin{proof}
The objective difference is
\[
    J_{\alpha,\eta}(\Path_1;\Belief_t)-J_{\alpha,\eta}(\Path_2;\Belief_t)=\Delta B_{\alpha}-\eta\Delta W_{\alpha}.
\]
Since \(\Delta W_{\alpha}>0\), the sign of this expression changes at \(\eta=\Delta B_{\alpha}/\Delta W_{\alpha}\). The stated preference regions follow immediately.
\qed
\end{proof}
This result gives the basic regime boundary. A high-information/high-risk path is selected when its information advantage outweighs the perceived-risk penalty, while a lower-risk path is selected once risk sensitivity is sufficiently large. It also shows why pure Behavioral Information Gain is not necessarily risk-aware: without the explicit risk term, physical path risk cannot trade off against information gain.

For fixed \(\eta\), changing \(\alpha\) induces a second type of switching because both the Behavioral Information value and the perceived-risk term depend on probability weighting. The corresponding \(\alpha\)-switching boundary is defined implicitly by
\begin{equation}
    \InfoB(\Path_1;\Belief_t)-\InfoB(\Path_2;\Belief_t)=\eta\left[w_{\alpha}\!\left(\Risk(\Path_1;\Belief_t)\right)-w_{\alpha}\!\left(\Risk(\Path_2;\Belief_t)\right)\right].
    \label{eq:alpha_star}
\end{equation}
Closed-form solutions are generally unavailable, but \cref{eq:alpha_star} shows that changing \(\alpha\) can move the system across preference regimes even when the candidate paths are fixed. \Cref{fig:theory_geometry}(a) shows these regimes in the \((\alpha,\eta)\) plane, while \Cref{fig:theory_geometry}(b) shows the corresponding nonmonotone switching curve \(\eta^*(\alpha)\). Together, these plots illustrate the valuation-layer claim that path preference changes because the same information-risk tradeoff is valued differently.

\begin{figure}[t]
    \centering
    \begin{subfigure}{0.32\linewidth}
        \centering
        \includegraphics[width=\linewidth]{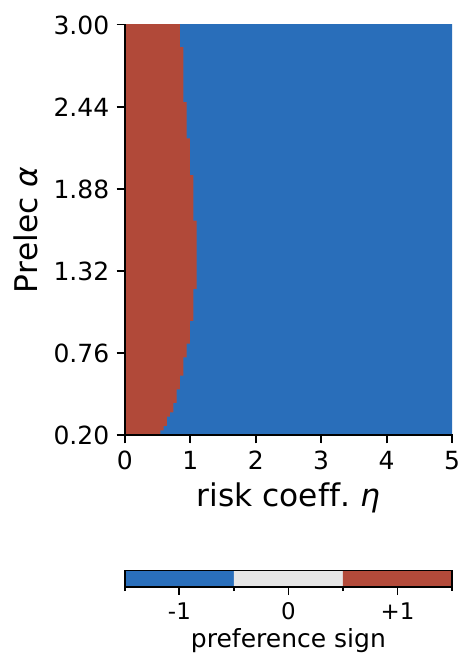}
        \label{fig:preference_map}
    \end{subfigure}
    \hfill
    \begin{subfigure}{0.32\linewidth}
        \centering
        \includegraphics[width=\linewidth]{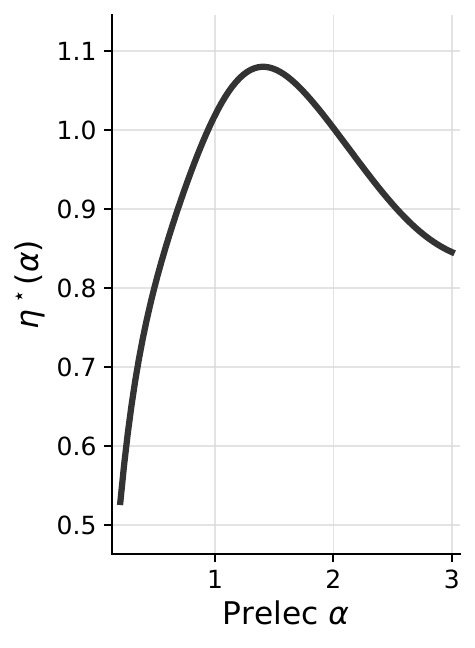}
        \label{fig:eta_star}
    \end{subfigure}
    \hfill
    \begin{subfigure}{0.32\linewidth}
        \centering
        \includegraphics[width=\linewidth]{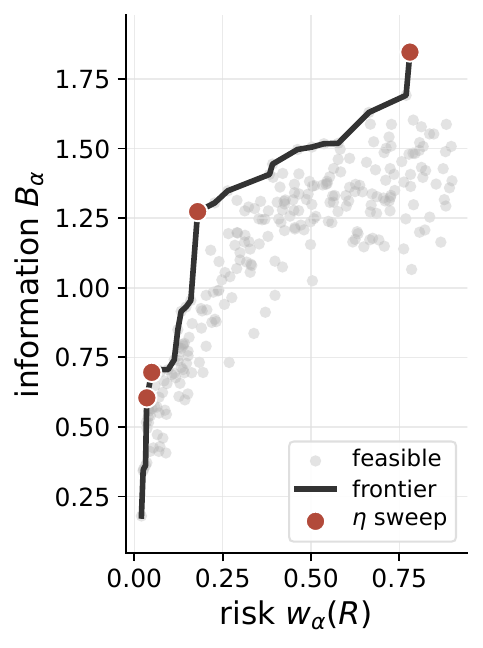}
        \label{fig:pareto_geometry}
    \end{subfigure}
    \caption{Theoretical geometry of valuation-induced path selection. (a) Preference regions in the \((\alpha,\eta)\) plane, where positive values select the high-information/high-risk path and negative values select the lower-risk path. (b) The nonmonotone switching boundary \(\eta^*(\alpha)\). (c) For fixed \(\alpha\), varying \(\eta\) selects supported points of the transformed information-risk tradeoff set.}
    \label{fig:theory_geometry}
\end{figure}

\subsection{Transformed Information-Risk Frontier}
\label{subsec:5.2}
The switching analysis compares two paths. We now extend the same idea to the full feasible path set. For fixed \(\alpha\), each candidate path induces a point in a transformed risk-information plane,
\[
    \Path \mapsto
    \bigl(
        w_{\alpha}\bigl(\Risk(\Path;\Belief_t)\bigr),
        \InfoB(\Path;\Belief_t)
    \bigr),
\]
where the first coordinate is perceived path risk and the second is Behavioral Information Gain. The risk-augmented Behavioral objective ranks these points by a linear tradeoff in the transformed space.

\begin{proposition}[Transformed frontier scalarization]
\label{prop:pareto}
For fixed \(\alpha\), define the transformed tradeoff set
\begin{equation}
\begin{aligned}
    \mathcal{P}_{\alpha}
    =
    \big\{
    (&w_{\alpha}\!\left(\Risk(\Path;\Belief_t)\right),
    \InfoB(\Path;\Belief_t)) :
    &\Path\in\Paths(x_t)
    \big\}.
\end{aligned}
\label{eq:transformed_tradeoff_set}
\end{equation}
If \(\Path^*_{\alpha,\eta}\) is a unique maximizer of \(J_{\alpha,\eta}\), then its corresponding point in \(\mathcal{P}_{\alpha}\) is a supported Pareto point of the transformed tradeoff set. Varying \(\eta\) changes the supporting slope, while varying \(\alpha\) changes the transformed geometry of the tradeoff set.
\end{proposition}

\begin{proof}
For fixed \(\alpha\), define the transformed coordinates
\[
    (r_{\alpha},b)
    =
    \left(
    w_{\alpha}\!\left(\Risk(\Path;\Belief_t)\right),
    \InfoB(\Path;\Belief_t)
    \right).
\]
The objective becomes
\begin{equation}
    J_{\alpha,\eta}=b-\eta r_{\alpha}.
\end{equation}
Thus, maximizing \(J_{\alpha,\eta}\) is equivalent to maximizing a linear scalarization over \(\mathcal{P}_{\alpha}\). The level sets are lines in the transformed risk-information plane with slope \(\eta\), so any unique maximizer is a supported point of the transformed Pareto frontier. Changing \(\eta\) changes the supporting slope; changing \(\alpha\) changes the transformed coordinates and therefore the geometry of the frontier.
\qed
\end{proof}

\Cref{fig:theory_geometry}(c) illustrates this interpretation. The feasible paths are fixed, but changing \(\eta\) selects different supporting points on the transformed frontier. Changing \(\alpha\) changes the perceived risk coordinate and, through \(\InfoB\), the information coordinate. This explains how Behavioral valuation can produce conservative, intermediate, and aggressive regimes without modifying the Bayesian filter or planner.

\begin{remark}[Supported frontier points]
Linear scalarization recovers supported Pareto points in the transformed space. Nonconvex portions of the tradeoff set may require constrained or explicitly multiobjective methods. Thus, \Cref{fig:theory_geometry}(c) should be interpreted as the geometry induced by the proposed valuation objective, not as an exhaustive enumeration of all Pareto-optimal paths. {In strongly nonconvex environments, unsupported Pareto points may represent useful intermediate information--risk tradeoffs that cannot be recovered by varying $\eta$ alone.}
\end{remark}
\subsection{Implications of the Valuation Layer}

The preceding analysis gives three implications that structure the empirical study. First, any change in selected paths can be attributed to valuation, as the feasible path set, Bayesian update, sensor model, and physical risk model are fixed, while only the scalar ranking objective changes. This makes the experiments a controlled test of objective-layer design rather than a comparison of different planners or filters. Second, the two valuation parameters have distinct geometric roles. For a fixed \(\alpha\), \(\eta\) changes the supporting slope of the information-risk tradeoff; changing \(\alpha\) changes the transformed geometry itself by altering perceived risk and Behavioral Information values. Third, the objective changes rankings only when information and risk compete. If one path is at least as informative and no riskier than another, monotonicity of \(w_{\alpha}\) prevents the dominated path from being preferred. Thus, the proposed Behavioral valuation provides an interpretable mechanism for selecting mission-dependent operating points along an information-risk frontier.

\section{Experimental Design}
\label{sec:experiments}

The experiments isolate the effect of the valuation objective. All methods use the same finite-horizon dynamic-programming planner, sensor model, Bayesian belief update, physical risk model, and feasible grid paths. At each deployment, the planner evaluates candidate paths using the shared information and risk models, ranks them using the chosen objective, executes the highest-valued path, and updates the belief from the observations collected. We compare the risk-neutral Shannon objective in~\eqref{eq:shannon_baseline}, pure Behavioral Information Gain obtained from~\eqref{eq:risk_aug_behavioral} with \(\eta=0\), risk-augmented Behavioral valuation in~\eqref{eq:risk_aug_behavioral}, risk-penalized Shannon planning in~\eqref{eq:risk_penalized_shannon}, and chance-constrained Shannon planning in~\eqref{eq:chance_shannon}. This separates the effects of changing the information functional, adding an explicit risk penalty, and applying nonlinear probability weighting to physical risk.
\subsection{Failure-Truncated Execution}
Execution is failure-truncated. After a path is selected, the robot follows it sequentially. At each visited cell, failure may occur due to the true hazard and lethality, or independently due to a mechanical malfunction. If failure occurs, execution stops, later cells on the planned path are not observed, and the next deployment restarts from the base. Thus, risk affects not only safety but also realized information gain. We record total failures and decompose them into hazard-induced and malfunction-induced losses. 
This separates losses caused by environmental hazard exposure from malfunctions.

\subsection{Experimental Stages and Metrics}
\label{subsec:6.2}
We evaluate the objectives in two stages. Stage 1 is the core frontier study, evaluating the full \((\alpha,\eta)\) grid on clustered and frontier-risk hazard structures. Stage 2 tests robustness using a reduced objective set across additional hazard structures, densities, and lethality values. Table~\ref{tab:experimental_stages} summarizes the design: Stage 1 includes \(13{,}600\) scenario--method--seed runs and Stage 2 includes \(36{,}000\), with each run consisting of 40 deployments. We report entropy reduction, expected hazard risk, expected total risk, sampled losses, and Pareto nondominance. Entropy reduction is the decrease in total Shannon hazard-belief entropy {summed over all map cells, reported in \texttt{nats} and not normalized by map size}; expected hazard risk is the cumulative planner risk; expected total risk also includes independent malfunction; and sampled losses are separated into hazard-induced and malfunction-induced losses. Because tight risk-aware baselines can avoid exploration altogether, policies with near-zero entropy reduction, near-base behavior, or frequent infeasible/no-op paths are marked degenerate; the main analysis focuses on the filtered nondegenerate frontier.

\begin{table}[t]
\centering
\caption{Summary of experimental stages.}
\label{tab:experimental_stages}
\begin{tabular}{lcc}
\toprule
 & Stage 1 & Stage 2 \\
\midrule
Purpose & Core frontier analysis & Robustness \\
Grid size & \(15\times15\) & \(15\times15\) \\
Hazard structures & 2 & 4 \\
Densities & \(0.10\) & \(0.05,0.10,0.20\) \\
Lethality & \(0.70\) & \(0.5,0.7,0.9\) \\
Sensor \(p_{\mathrm{TP}}\) & \(0.90\) & \(0.90\) \\
Sensor \(p_{\mathrm{FP}}\) & \(0.05\) & \(0.05\) \\
Malfunction rate & \(0.05\) & \(0.05\) \\
Seeds & 100 & 50 \\
Deployments & 40 & 40 \\
Scenario--method--seed runs & \(13{,}600\) & \(36{,}000\) \\
Planner evaluations & \(544{,}000\) & \(1{,}440{,}000\) \\
\bottomrule
\end{tabular}
\end{table}

\section{Results}
\label{sec:results}

We evaluate the proposed valuation-layer view under a fixed planning substrate. The results address three questions suggested by the theoretical analysis---\textit{Does changing the valuation objective reshape the information-risk frontier? Does failure-truncated execution make risk awareness relevant to realized information gain? And do risk-augmented Behavioral objectives remain Pareto-competitive across different hazard structures, densities, and lethality values?}

\Cref{fig:valuation_case_study} first shows a representative rollout from the simulation environment. The example illustrates the mechanism behind the aggregate results: under the same world, belief, sensing model, risk model, candidate paths, and planner, changing only the valuation objective changes the selected path and the realized sensing outcome under failure-truncated execution.

\begin{figure}[t]
    \centering
    \includegraphics[width=\linewidth]{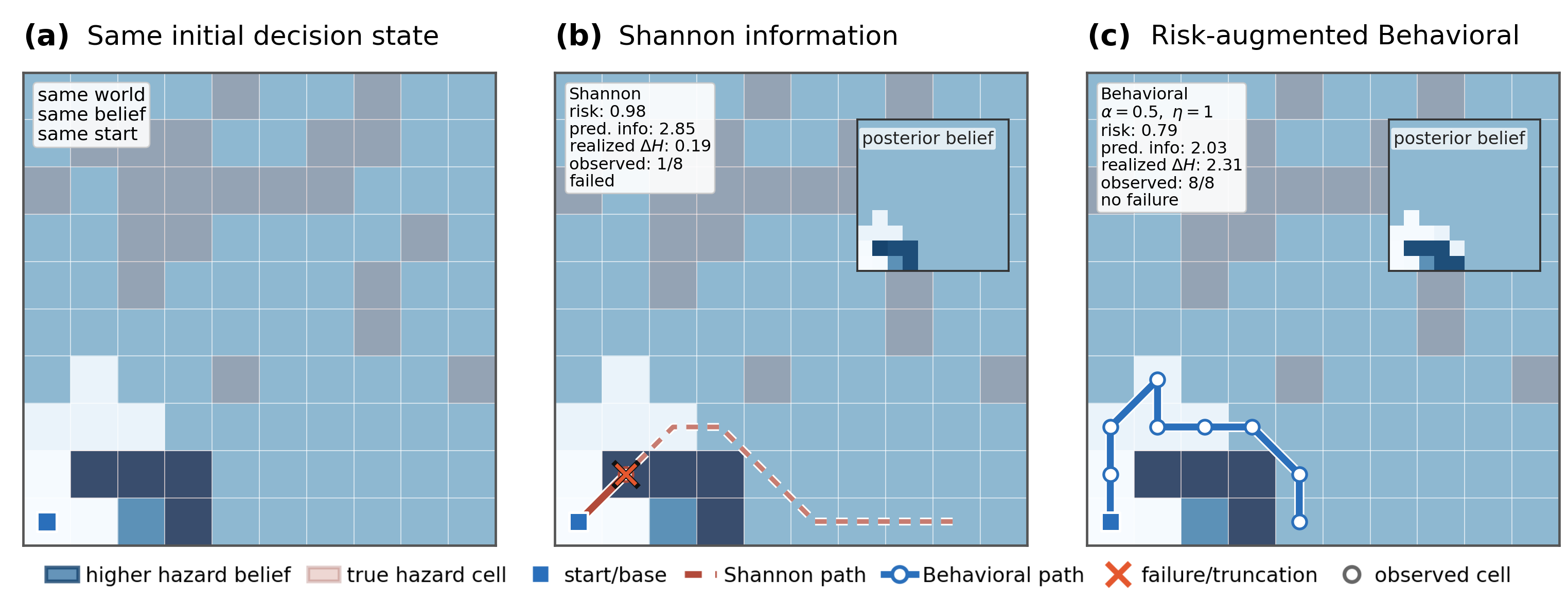}
    \caption{All panels share the same world, initial belief, start state, sensor model, risk model, candidate paths, and planner. Shannon information planning selects a high-information, high-risk path that fails early, collecting \(1/8\) planned observations, while the proposed Behavioral valuation selects a lower-risk path that completes execution and collects \(8/8\) observations.}
    \label{fig:valuation_case_study}
\end{figure}

\subsection{Valuation Reshapes the Information-Risk Frontier}
\label{subsec:7.1}
\Cref{fig:stage1_results}(a) shows the filtered nondegenerate Stage 1 information-risk frontier. Each point is produced by a different valuation objective under the same planning, sensing, Bayesian update, and physical risk models. The frontier shows that hazardous exploration does not yield a single dominant objective, but a set of mission-dependent operating points. Shannon occupies the high-information/high-risk region, while a linear risk penalty reduces hazard exposure with modest loss in entropy reduction. Risk-augmented Behavioral valuation adds nonlinear intermediate and conservative regimes along this frontier.

\Cref{fig:stage1_results}(b) shows how these regimes arise within the Behavioral family. The parameter \(\eta\) controls the perceived-risk penalty, while \(\alpha\) changes how physical risk is distorted before penalization. Low risk weight yields aggressive information seeking, moderate settings produce favorable information-per-risk tradeoffs, and larger risk weights, especially with small \(\alpha\), produce conservative or suppressed-exploration behavior. Thus, Behavioral valuation provides an interpretable mechanism for selecting mission-dependent operating points along the information-risk frontier. {For a new environment, we do not prescribe a universal $(\alpha,\eta)$ pair; instead, these parameters should be selected according to mission-level risk tolerance. In practice, $\eta$ controls the strength of the information--risk tradeoff, while $\alpha$ controls the nonlinear valuation of risk probabilities, so representative simulations or prior mission data can be used to identify acceptable operating points.}

\begin{figure}[t]
    \centering
    \begin{subfigure}{0.49\linewidth}
        \centering
        \includegraphics[width=\linewidth]{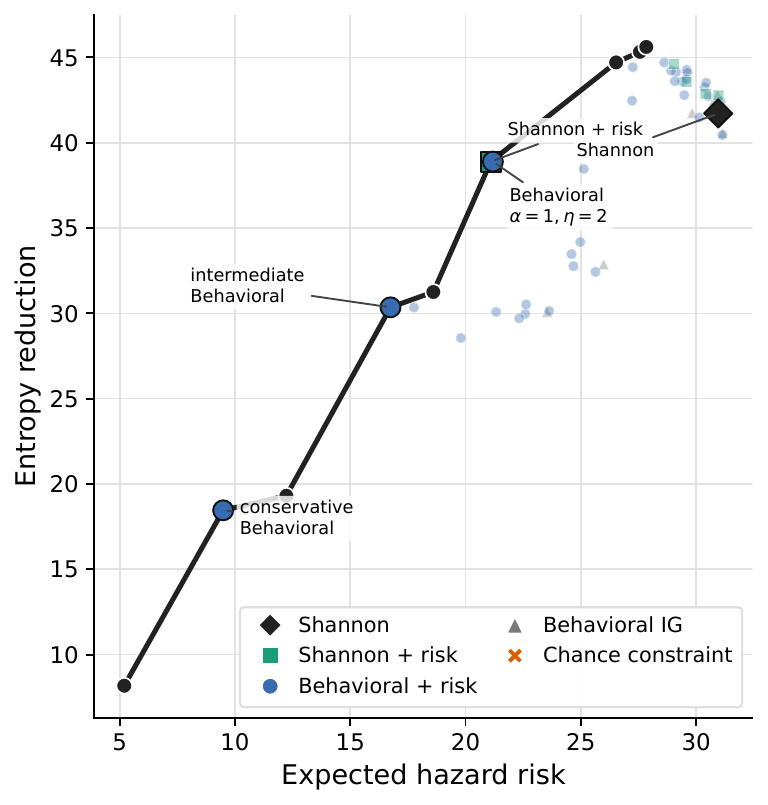}
        \caption{Information-risk frontier}
        \label{fig:stage1_pareto}
    \end{subfigure}
    \hfill
    \begin{subfigure}{0.49\linewidth}
        \centering
        \includegraphics[width=\linewidth]{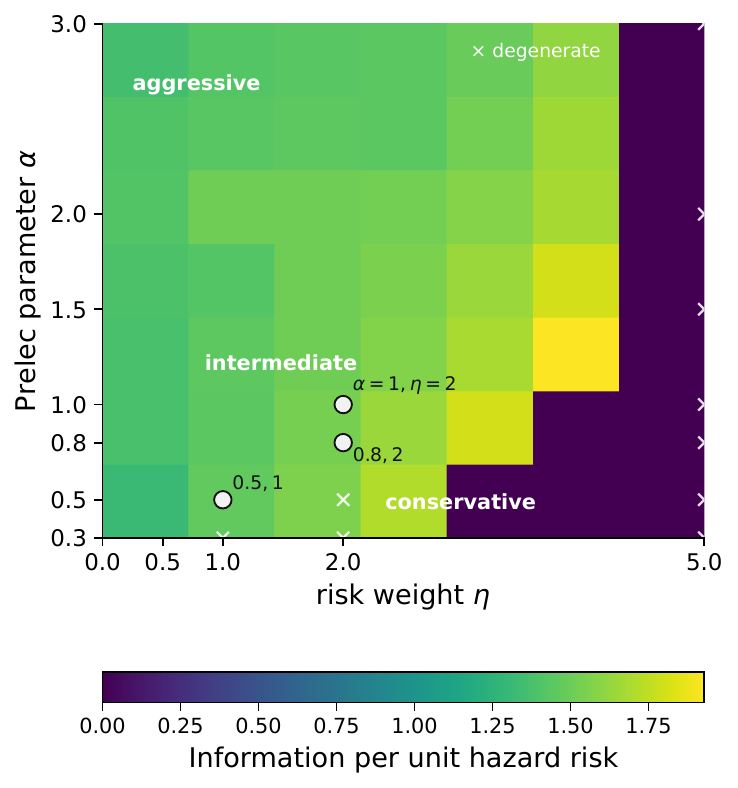}
        \caption{Behavioral parameter regimes}
        \label{fig:alpha_eta_efficiency}
    \end{subfigure}
    \caption{Stage 1 valuation effects. (a) Changing only the valuation objective moves policies from high-information/high-risk Shannon behavior toward intermediate and conservative risk-aware regimes. (b) The \((\alpha,\eta)\) sweep organizes Behavioral valuation into aggressive, intermediate, conservative, and suppressed-exploration regimes.}
    \label{fig:stage1_results}
\end{figure}
\begin{table}[t]
\centering
\caption{Stage 1 aggregate results.}
\label{tab:stage1}
\begin{tabular}{lrrrrrr}
\toprule
Method & Entropy & Hazard risk & Total risk & Losses & Hazard loss & Malf. loss \\
\midrule
Shannon & 41.72 & 30.96 & 33.00 & 21.21 & 14.62 & 7.36 \\
Risk-pen. Shannon \((\eta{=}2)\) & 38.90 & 21.10 & 25.37 & 16.22 & 8.00 & 8.55 \\
Beh. risk \((\alpha{=}1,\eta{=}2)\) & 38.89 & 21.19 & 25.44 & 16.11 & 7.97 & 8.53 \\
Beh. risk \((\alpha{=}0.8,\eta{=}2)\) & 18.46 & 9.49 & 16.39 & 11.81 & 3.15 & 8.79 \\
Chance Shannon \((\delta{=}0.35)\) & 8.18 & 5.20 & 13.07 & 9.96 & 1.27 & 8.76 \\
\bottomrule
\end{tabular}
\end{table}

\Cref{tab:stage1} reports representative Stage 1 aggregate results. Shannon provides the high-information reference point, while risk-penalized Shannon is a strong linear-risk baseline. This strength is expected because it represents the linear-risk member of the same valuation-layer design space. Behavioral valuation is therefore not intended to dominate this baseline universally, but to expose nonlinear intermediate and conservative regimes that a single linear risk penalty may not capture. Failure truncation explains why these intermediate frontier points are useful: high predicted information may not be realized if the robot fails early, while overly conservative policies suppress exploration. 

\subsection{Behavioral Valuation Is Pareto-Competitive Across Scenarios}
\begin{figure}[t]
    \centering
    \includegraphics[width=0.8\linewidth]{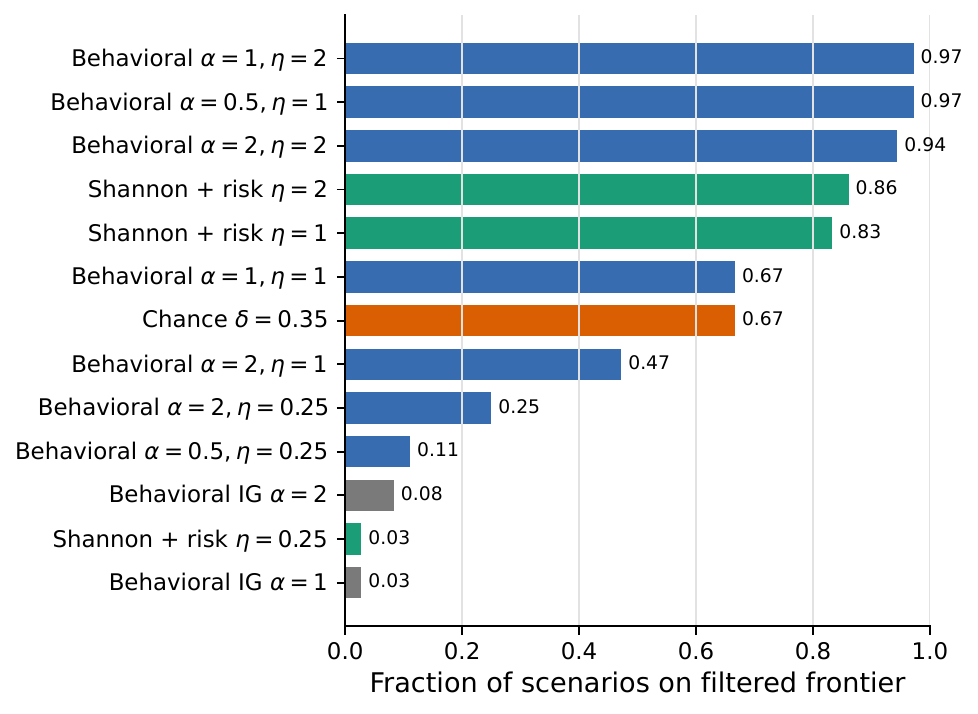}
    \caption{Stage 2 scenario-level frontier frequency across hazard structures, densities, and lethality values. Risk-augmented Behavioral settings repeatedly appear on the filtered nondegenerate frontier, while risk-penalized Shannon remains a strong baseline.}
    \label{fig:stage2_nondominated}
\end{figure}

Stage 2 tests whether the Stage 1 frontier behavior persists across hazard structures, densities, and lethality values. We measure how often each objective appears on the filtered nondegenerate scenario-level frontier, which captures whether it repeatedly produces useful information-risk tradeoffs across heterogeneous risk landscapes. \Cref{fig:stage2_nondominated} shows that several risk-augmented Behavioral settings appear on the frontier across many Stage 2 scenarios, while risk-penalized Shannon also appears frequently. Thus, Behavioral valuation does not rely on isolated favorable cases; it complements standard risk penalties by adding interpretable Pareto-competitive operating points. Table~\ref{tab:stage2} reports aggregate trends consistent with this analysis: Shannon achieves high entropy reduction with high hazard exposure, risk-penalized Shannon remains a strong high-information risk-aware baseline, and risk-augmented Behavioral valuation contributes both nearby high-information points and more conservative lower-risk regimes.

\begin{table}[t]
\centering
\caption{Stage 2 robustness aggregate results across hazard structures, densities, and lethality values.}
\label{tab:stage2}
\begin{tabular}{lrrrrrr}
\toprule
Method & Entropy & Hazard risk & Total risk & Losses & Hazard loss & Malf. loss \\
\midrule
Shannon & 44.15 & 30.62 & 32.74 & 19.94 & 12.87 & 7.72 \\
Risk-pen. Shannon \((\eta{=}2)\) & 41.66 & 21.68 & 25.82 & 15.20 & 7.20 & 8.32 \\
Beh. risk \((\alpha{=}1,\eta{=}2)\) & 42.10 & 21.79 & 25.91 & 14.80 & 7.25 & 7.92 \\
Beh. risk \((\alpha{=}0.5,\eta{=}1)\) & 31.81 & 16.22 & 21.60 & 13.39 & 4.83 & 8.81 \\
\bottomrule
\end{tabular}
\end{table}

\section{Conclusion}
\label{sec:conclusion}
This paper introduced a valuation-layer framework for hazardous robotic exploration, where acquiring information may cause failure and truncate future sensing. By holding the Bayesian update, sensor model, physical risk model, feasible paths, and planner fixed while changing only the path-ranking objective, we isolate valuation as a distinct design layer for reasoning about when information is worth the risk. We instantiated this framework with a risk-augmented Behavioral Information objective based on Prelec probability weighting, where $\eta$ controls risk sensitivity and $\alpha$ controls nonlinear risk valuation. Theoretical analysis showed that these parameters induce path-switching boundaries and transform the information-risk Pareto geometry. Large-scale experiments showed that valuation reshapes the information-risk frontier, that risk awareness can preserve realized information gain by avoiding premature failure, and that Behavioral valuation provides Pareto-competitive aggressive, intermediate, and conservative operating regimes alongside strong Shannon-based baselines.

\bibliographystyle{ieeetr}
\bibliography{references}

\end{document}